\documentclass[sigconf]{acmart}

\usepackage{bm}
\usepackage{amsmath}
\usepackage{amsfonts}
\usepackage{amssymb}
\usepackage{amsthm}
\usepackage[linesnumbered,vlined,ruled]{algorithm2e}
\usepackage{multirow}
\usepackage{caption}
\usepackage{subcaption}
\usepackage{etoolbox}
\makeatletter
\patchcmd{\@algocf@start}{%
  \begin{lrbox}{\algocf@algobox}%
}{%
  \rule{0.\textwidth}{\z@}%
  \begin{lrbox}{\algocf@algobox}%
  \begin{minipage}{0.52\textwidth}%
}{}{}
\patchcmd{\@algocf@finish}{%
  \end{lrbox}%
}{%
  \end{minipage}%
  \end{lrbox}%
}{}{}
\makeatother
\usepackage{enumitem}
\setlist[itemize,1]{leftmargin=2em}
\newtheorem{theorem}{Theorem}

\copyrightyear{2019}
\acmYear{2019} 
\setcopyright{none}
\acmConference[KDD '19]{The 25th ACM SIGKDD Conference on Knowledge Discovery and Data Mining}{August 4--8, 2019}{Anchorage, AK, USA}
\acmBooktitle{The 25th ACM SIGKDD Conference on Knowledge Discovery and Data Mining (KDD '19), August 4--8, 2019, Anchorage, AK, USA}
\acmPrice{15.00}
\acmDOI{10.1145/3292500.3330836}
\acmISBN{978-1-4503-6201-6/19/08}

\begin{document}
 
\title{Knowledge-aware Graph Neural Networks with Label Smoothness Regularization for Recommender Systems}


\author{Hongwei Wang}
\affiliation{Stanford University}
\email{hongweiw@cs.stanford.edu}

\author{Fuzheng Zhang}
\affiliation{Meituan-Dianping Group}
\email{zhangfuzheng@meituan.com}

\author{Mengdi Zhang}
\affiliation{Meituan-Dianping Group}
\email{zhangmengdi02@meituan.com}

\author{Jure Leskovec}
\affiliation{Stanford University}
\email{jure@cs.stanford.edu}

\author{Miao Zhao, Wenjie Li}
\affiliation{Hong Kong Polytechnic University}
\email{{csmiaozhao,cswjli}@comp.polyu.edu.hk}

\author{Zhongyuan Wang}
\affiliation{Meituan-Dianping Group}
\email{wangzhongyuan02@meituan.com}

\begin{abstract}
	Knowledge graphs capture structured information and relations between a set of entities or items.
	As such knowledge graphs represent an attractive source of information that could help improve recommender systems.
	However, existing approaches in this domain rely on manual feature engineering and do not allow for an end-to-end training.
	Here we propose \textit{Knowledge-aware Graph Neural Networks with Label Smoothness regularization} (KGNN-LS) to provide better recommendations.
	Conceptually, our approach computes user-specific item embeddings by first applying a trainable function that identifies important knowledge graph relationships for a given user. 
	This way we transform the knowledge graph into a user-specific weighted graph and then apply a graph neural network to compute personalized item embeddings.
	To provide better inductive bias, we rely on \textit{label smoothness} assumption, which posits that adjacent items in the knowledge graph are likely to have similar user relevance labels/scores.
	Label smoothness provides regularization over the edge weights and we prove that it is equivalent to a label propagation scheme on a graph.
	We also develop an efficient implementation that shows strong scalability with respect to the knowledge graph size.
	Experiments on four datasets show that our method outperforms state of the art baselines.
	KGNN-LS also achieves strong performance in cold-start scenarios where user-item interactions are sparse.
\end{abstract}


\maketitle

\section{Introduction}
    Recommender systems are widely used in Internet applications to meet user's personalized interests and alleviate the information overload \cite{covington2016deep, wang2018billion, ying2018graph}.
	Traditional recommender systems that are based on collaborative filtering \cite{koren2009matrix,wang2017joint} usually suffer from the cold-start problem and have trouble recommending brand new items that have not yet been heavily explored by the users.
	The sparsity issue can be addressed by introducing additional sources of information such as user/item profiles \cite{wang2018shine} or social networks \cite{wang2017joint}.
	
	Knowledge graphs (KGs) capture structured information and relations between a set of entities \cite{zhang2016collaborative,wang2018dkn,huang2018improving,yu2014personalized,zhao2017meta,hu2018leveraging,wang2018ripple,sun2018recurrent,wang2019multi,wang2019knowledge,wang2019exploring}.
	KGs are heterogeneous graphs in which nodes correspond to \textit{entities} (e.g., items or products, as well as their properties and characteristics) and edges correspond to \textit{relations}.	KGs provide connectivity information between items via different types of relations and thus capture semantic relatedness between the items.
	
	The core challenge in utilizing KGs in recommender systems is to learn how to capture \textit{user-specific} item-item relatedness captured by the KG.
	Existing KG-aware recommender systems can be classified into path-based methods \cite{yu2014personalized,zhao2017meta,hu2018leveraging}, embedding-based methods \cite{zhang2016collaborative,wang2018dkn,huang2018improving,wang2019multi}, and hybrid methods \cite{wang2018ripple,sun2018recurrent,wang2019knowledge}.
	However, these approaches rely on manual feature engineering, are unable to perform end-to-end training, and have poor scalability.
	Graph Neural Networks (GNNs), which aggregate node feature information from node's local network neighborhood using neural networks, represent a promising advancement in graph-based representation learning \cite{bruna2014spectral,defferrard2016convolutional,kipf2017semi,duvenaud2015convolutional,niepert2016learning,hamilton2017inductive}.
	Recently, several works developed GNNs architecture for recommender systems \cite{ying2018graph,monti2017geometric,van2017graph,wu2018graph,wang2019knowledge}, but these approaches are mostly designed for \textit{homogeneous} bipartite user-item interaction graphs or user-/item-similarity graphs.
	It remains an open question how to extend GNNs architecture to \textit{heterogeneous} knowledge graphs.

	In this paper, we develop \textit{Knowledge-aware Graph Neural Networks with Label Smoothness regularization} (KGNN-LS) that extends GNNs architecture to knowledge graphs to simultaneously capture semantic relationships between the items as well as personalized user preferences and interests.
	To account for the relational heterogeneity in KGs, similar to \cite{wang2019knowledge}, we use a trainable and personalized relation scoring function that transforms the KG into a user-specific weighted graph, which characterizes both the semantic information of the KG as well as user's personalized interests.
	For example, in the movie recommendation setting the relation scoring function could learn that a given user really cares about ``director'' relation between movies and persons, while somebody else may care more about the ``lead actor'' relation.
	Using this personalized weighted graph, we then apply a graph neural network that for every item node computes its embedding by aggregating node feature information over the local network neighborhood of the item node.
	This way the embedding of each item captures it's local KG structure in a user-personalized way.
	
	A significant difference between our approach and traditional GNNs is that the edge weights in the graph are not given as input.
	We set them using user-specific relation scoring function that is trained in a supervised fashion.
	However, the added flexibility of edge weights makes the learning process prone to overfitting, since the only source of supervised signal for the relation scoring function is coming from user-item interactions (which are sparse in general).
	To remedy this problem, we develop a technique for regularization of edge weights during the learning process, which leads to better generalization.
	We develop an approach based on \textit{label smoothness} \cite{zhu2003semi,zhang2007hyperparameter}, which assumes that adjacent entities in the KG are likely to have similar user relevancy labels/scores.
	In our context this assumption means that users tend to have similar preferences to items that are nearby in the KG.
	We prove that label smoothness regularization is equivalent to \textit{label propagation} and we design a \textit{leave-one-out} loss function for label propagation to provide extra supervised signal for learning the edge scoring function.
	We show that the knowledge-aware graph neural networks and label smoothness regularization can be unified under the same framework, where label smoothness can be seen as a natural choice of regularization on knowledge-aware graph neural networks.
	
	We apply the proposed method to four real-world datasets of movie, book, music, and restaurant recommendations, in which the first three datasets are public datasets and the last is from Meituan-Dianping Group.
	Experiments show that our method achieves significant gains over state-of-the-art methods in recommendation accuracy.
	We also show that our method maintains strong recommendation performance in the cold-start scenarios where user-item interactions are sparse.

\section{Related Work}
	
	\subsection{Graph Neural Networks}
	\label{sec:gcn_rs}
		Graph Neural Networks (or Graph Convolutional Neural Networks, GCNs) aim to generalize convolutional neural networks to non-Euclidean domains (such as graphs) for robust feature learning.
		Bruna et al. \cite{bruna2014spectral} define the convolution in Fourier domain and calculate the eigendecomposition of the graph Laplacian,
		Defferrard et al. \cite{defferrard2016convolutional} approximate the convolutional filters by Chebyshev expansion of the graph Laplacian,
		and Kipf et al. \cite{kipf2017semi} propose a convolutional architecture via a first-order approximation.
		In contrast to these \textit{spectral} GCNs, \textit{non-spectral} GCNs operate on the graph directly and apply ``convolution'' (i.e., weighted average) to local neighbors of a node \cite{duvenaud2015convolutional,niepert2016learning,hamilton2017inductive}.
		
		Recently, researchers also deployed GNNs in recommender systems: PinSage \cite{ying2018graph} applies GNNs to the pin-board bipartite graph in Pinterest.
		Monti et al. \cite{monti2017geometric} and Berg et al. \cite{van2017graph} model recommender systems as matrix completion and design GNNs for representation learning on user-item bipartite graphs.
		Wu et al. \cite{wu2018graph} use GNNs on user/item structure graphs to learn user/item representations.
		The difference between these works and ours is that they are all designed for homogeneous bipartite graphs or user/item-similarity graphs where GNNs can be used directly, while here we investigate GNNs for heterogeneous KGs.
		Wang et al. \cite{wang2019knowledge} use GCNs in KGs for recommendation, but simply applying GCNs to KGs without proper regularization is prone to overfitting and leads to performance degradation as we will show later.
		Schlichtkrull et al. also propose using GNNs to model KGs \cite{schlichtkrull2018modeling}, but not for the purpose of recommendations.

	\subsection{Semi-supervised Learning on Graphs}
		The goal of graph-based semi-supervised learning is to correctly label all nodes in a graph given that only a few nodes are labeled.
		Prior work often makes assumptions on the distribution of labels over the graph, and one common assumption is smooth variation of labels of nodes across the graph.
		Based on different settings of edge weights in the input graph, these methods are classified as:
		(1) Edge weights are assumed to be given as input and therefore fixed \cite{zhu2003semi,zhou2004learning,baluja2008video};
		(2) Edge weights are parameterized and therefore learnable \cite{zhang2007hyperparameter,wang2008label,karasuyama2013manifold}.
		Inspired by these methods, we design a module of label smoothness regularization in our proposed model.
		The major distinction of our work is that the label smoothness constraint is not used for semi-supervised learning on graphs, but serves as regularization to assist the learning of edge weights and achieves better generalization for recommender systems.
		
	\begin{figure*}[t]
		\centering
		\includegraphics[width=0.95\textwidth]{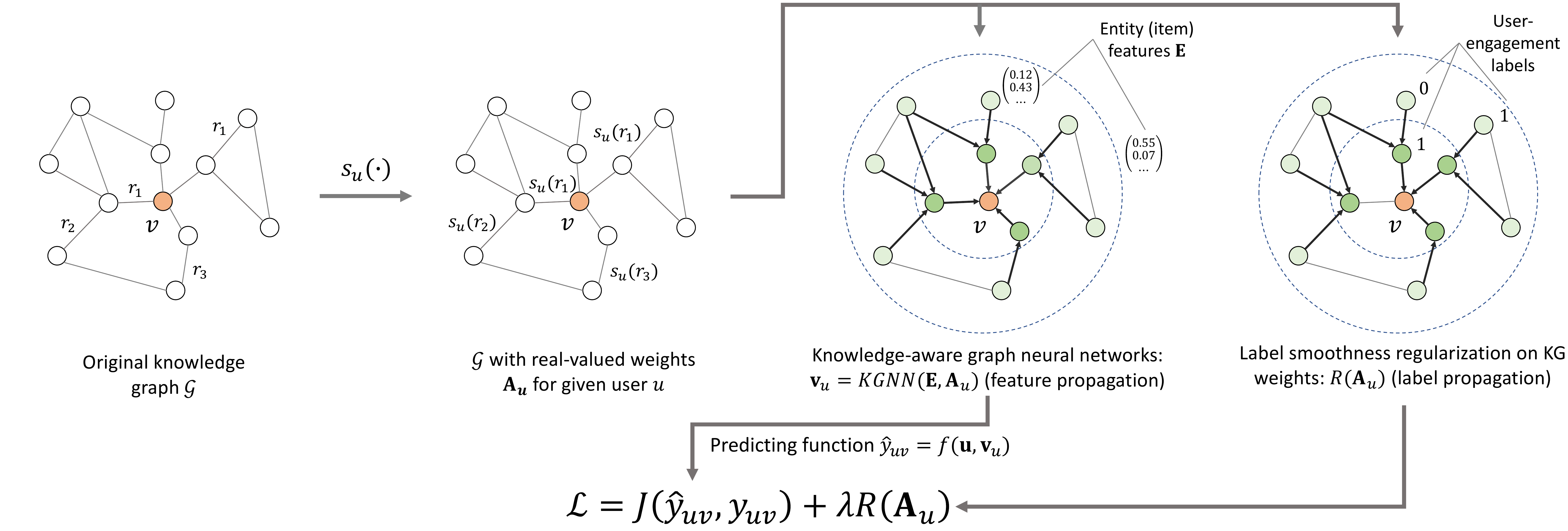}
		\caption{Overview of our proposed KGNN-LS model. The original KG is first transformed into a user-specific weighted graph, on which we then perform feature propagation using a graph neural network with the label smoothness regularization. The two modules constitute the complete loss function $\mathcal L$.}
		\label{fig:framework}
	\end{figure*}

	\subsection{Recommendations with Knowledge Graphs}
	\label{sec:kg_rs}
		In general, existing KG-aware recommender systems can be classified into three categories:
		(1) \textit{Embedding-based methods} \cite{zhang2016collaborative,wang2018dkn,huang2018improving,wang2019multi} pre-process a KG with \textit{knowledge graph embedding} (KGE) \cite{wang2017knowledge} algorithms, then incorporate learned entity embeddings into recommendation.
		Embedding-based methods are highly flexible in utilizing KGs to assist recommender systems, but the KGE algorithms focus more on modeling rigorous semantic relatedness (e.g., TransE \cite{bordes2013translating} assumes $head + relation = tail$), which are more suitable for graph applications such as link prediction rather than recommendations.
		In addition, embedding-based methods usually lack an end-to-end way of training.
		(2) \textit{Path-based methods} \cite{yu2014personalized,zhao2017meta,hu2018leveraging} explore various patterns of connections among items in a KG (a.k.a meta-path or meta-graph) to provide additional guidance for recommendations.
		Path-based methods make use of KGs in a more intuitive way, but they rely heavily on manually designed meta-paths/meta-graphs, which are hard to tune in practice.
		(3) \textit{Hybrid methods} \cite{wang2018ripple,sun2018recurrent,wang2019knowledge} combine the above two categories and learn user/item embeddings by exploiting the structure of KGs.
		Our proposed model can be seen as an instance of hybrid methods.

\section{Problem Formulation}
	\label{section:pf}
		
	We begin by describing the KG-aware recommendations problem and introducing notation.
	In a typical recommendation scenario, we have a set of users $\mathcal U$ and a set of items $\mathcal V$.
	The user-item interaction matrix ${\bf Y}$ is defined according to users' implicit feedback, where $y_{uv} = 1$ indicates that user $u$ has engaged with item $v$, such as clicking, watching, or purchasing.
	We also have a knowledge graph $\mathcal G = \{(h, r, t)\}$ available, in which $h \in \mathcal E$, $r \in \mathcal R$, and $t \in \mathcal E$ denote the head, relation, and tail of a knowledge triple, $\mathcal E$ and $\mathcal R$ are the set of entities and relations in the knowledge graph, respectively.
	For example, the triple (\textit{The Silence of the Lambs}, \textit{film.film.star}, \textit{Anthony Hopkins}) states the fact that Anthony Hopkins is the leading actor in film ``The Silence of the Lambs''.
	In many recommendation scenarios, an item $v \in \mathcal V$ corresponds to an entity $e \in \mathcal E$ (e.g., item ``The Silence of the Lambs'' in MovieLens also appears in the knowledge graph as an entity).
	The set of entities $\mathcal E$ is composed from items $\mathcal V$ ($\mathcal V \subseteq \mathcal E$) as well as non-items $\mathcal E \backslash \mathcal V$ (e.g. nodes corresponding to item/product properties).
	Given user-item interaction matrix $\bf Y$ and knowledge graph $\mathcal G$, our task is to predict whether user $u$ has potential interest in item $v$ with which he/she has not engaged before.
	Specifically, we aim to learn a prediction function ${\hat y}_{uv} = \mathcal F(u, v | \Theta, \bf Y, \mathcal G)$, where ${\hat y}_{uv}$ denotes the probability that user $u$ will engage with item $v$, and $\Theta$ are model parameters of function $\mathcal F$.
	
	We list the key symbols used in this paper in Table \ref{table:notation}.
	
	\begin{table}[b]
		\vspace{-0.1in}
		\centering
		\setlength{\tabcolsep}{4pt}
		\begin{tabular}{c|c}
			\hline
			Symbol & Meaning \\
			\hline
			$\mathcal U = \{u_1, \cdots\}$ & Set of users \\
			$\mathcal V = \{v_1, \cdots\}$ & Set of items \\
			${\bf Y}$ & User-item interaction matrix \\
			$\mathcal G = (\mathcal E, \mathcal R)$ & Knowledge graph \\
			$\mathcal E = \{e_1, \cdots\}$ & Set of entities \\
			$\mathcal R = \{r_1, \cdots\}$ & Set of relations \\
			$\mathcal E \backslash \mathcal V$ & Set of non-item entities \\
			$s_u(r)$ & User-specific relation scoring function \\
			${\bf A}_u$ & Adjacency matrix of $\mathcal G$ w.r.t. user $u$ \\
			${\bf D}_u$ & Diagonal degree matrix of ${\bf A}_u$ \\
			${\bf E}$ & Raw entity feature \\
			${\bf H}^{(l)}, \ l = 0, ..., L-1$ & Entity representation in the $l$-th layer \\
			${\bf W}^{(l)}, \ l = 0, ..., L-1$ & Transformation matrix in the $l$-th layer \\
			$l_u(e), \ e \in \mathcal E$ & Item relevancy labeling function \\
			$l_u^*(e), \ e \in \mathcal E$ & Minimum-energy labeling function \\
			$\hat l_u(v), \ v \in \mathcal V$ & Predicted relevancy label for item $v$ \\
			$R({\bf A}_u)$ & Label smoothness regularization on ${\bf A}_u$ \\
			\hline
		\end{tabular}
		\vspace{0.05in}
		\caption{List of key symbols.}
		\label{table:notation}
		\vspace{-0.15in}
	\end{table}

\section{Our Approach}
	In this section, we first introduce knowledge-aware graph neural networks and label smoothness regularization, respectively, then we present the unified model.
	
	\subsection{Preliminaries: Knowledge-aware Graph Neural Networks}
		The first step of our approach is to transform a heterogeneous KG into a user-personalized weighted graph, which characterizes user's preferences.
		To this end, similar to \cite{wang2019knowledge}, we use a user-specific \textit{relation scoring function} $s_u(r)$ that provides the importance of relation $r$ for user $u$: $s_u(r) = g({\bf u}, {\bf r})$, where ${\bf u}$ and ${\bf r}$ are feature vectors of user $u$ and relation type $r$, respectively, and $g$ is a differentiable function such as inner product.
		Intuitively, $s_u(r)$ characterizes the importance of relation $r$ to user $u$.
		For example, a user may be more interested in movies that have the same \textit{director} with the movies he/she watched before, but another user may care more about the \textit{leading actor} of movies.
		
		Given user-specific relation scoring function $s_u(\cdot)$ of user $u$, knowledge graph $\mathcal G$ can therefore be transformed into a user-specific adjacency matrix ${\bf A}_u \in \mathbb R^{|\mathcal E| \times |\mathcal E|}$, in which the $(i, j)$-entry $A_u^{ij} = s_u(r_{e_i, e_j})$, and $r_{e_i, e_j}$ is the relation between entities $e_i$ and $e_j$ in $\mathcal G$.\footnote{In this work we treat $\mathcal G$ an undirected graph, so ${\bf A}_u$ is a symmetric matrix. If both triples $(h, r_1, t)$ and $(t, r_2, h)$ exist, we only consider one of $r_1$ and $r_2$. This is due to the fact that: (1) $r_1$ and $r_2$ are the inverse of each other and semantically related; (2) Treating ${\bf A}_u$ symmetric will greatly increase the matrix density.}
		$A_u^{ij}=0$ if there is no relation between $e_i$ and $e_j$.
		See the left two subfigures in Figure \ref{fig:framework} for illustration.
		We also denote the raw feature matrix of entities as ${\bf E} \in \mathbb R^{|\mathcal E| \times d_0}$, where $d_0$ is the dimension of raw entity features.
		Then we use multiple feed forward layers to update the entity representation matrix by aggregating representations of neighboring entities.
		Specifically, the layer-wise forward propagation can be expressed as
		\begin{equation}
		\label{eq:kgcn}
			{\bf H}_{l+1} = \sigma \left({\bf D}_u^{-1/2} {\bf A}_u {\bf D}_u^{-1/2} {\bf H}_l {\bf W}_l \right), \ l = 0, 1, \cdots, L-1.
		\end{equation}
		In Eq. (\ref{eq:kgcn}), ${\bf H}_l$ is the matrix of hidden representations of entities in layer $l$, and ${\bf H}_0 = {\bf E}$.
		${\bf A}_u$ is to aggregate representation vectors of neighboring entities.
		In this paper, we set ${\bf A}_u \leftarrow {\bf A}_u + \bf{I}$, i.e., adding self-connection to each entity, to ensure that old representation vector of the entity itself is taken into consideration when updating entity representations.
		${\bf D}_u$ is a diagonal degree matrix with entries $D_u^{ii} = \sum_j A_{u}^{ij}$, therefore, ${\bf D}_u^{-1/2}$ is used to normalize ${\bf A}_u$ and keep the entity representation matrix ${\bf H}_l$ stable.
		 ${\bf W}_l \in \mathbb R^{d_l \times d_{l+1}}$ is the layer-specific trainable weight matrix, $\sigma$ is a non-linear activation function, and $L$ is the number of layers.
		 
		A single GNN layer computes the representation of an entity via a transformed mixture of itself and its immediate neighbors in the KG.
		We can therefore naturally extend the model to multiple layers to explore users' potential interests in a broader and deeper way.
		The final output is ${\bf H}_L \in \mathbb R^{|\mathcal E| \times d_L}$, which is the entity representations that mix the initial features of themselves and their neighbors up to $L$ hops away.
		Finally, the predicted engagement probability of user $u$ with item $v$ is calculated by $\hat y_{uv} = f({\bf u}, {\bf v}_u)$, where ${\bf v}_u$ (i.e., the $v$-th row of ${\bf H}_L$) is the final representation vector of item $v$, and $f$ is a differentiable prediction function, for example, inner product or a multilayer perceptron.
		Note that ${\bf v}_u$ is user-specific since the adjacency matrix ${\bf A}_u$ is user-specific.
		Furthermore, note that the system is end-to-end trainable where the gradients flow from $f(\cdot)$ via GNN (parameter matrix $\bf W$) to $g(\cdot)$ and eventually to representations of users $u$ and items $v$.

	\begin{figure*}
			\centering
			\begin{subfigure}[b]{0.9\textwidth}
   				\includegraphics[width=\textwidth]{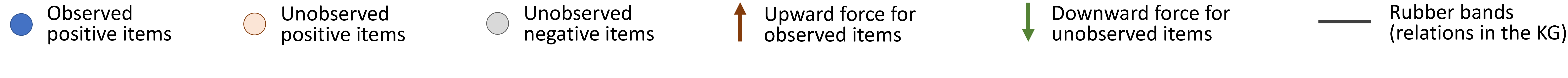}
			\end{subfigure}
			\hfill
			\begin{subfigure}[b]{0.19\textwidth}
   				\includegraphics[width=\textwidth]{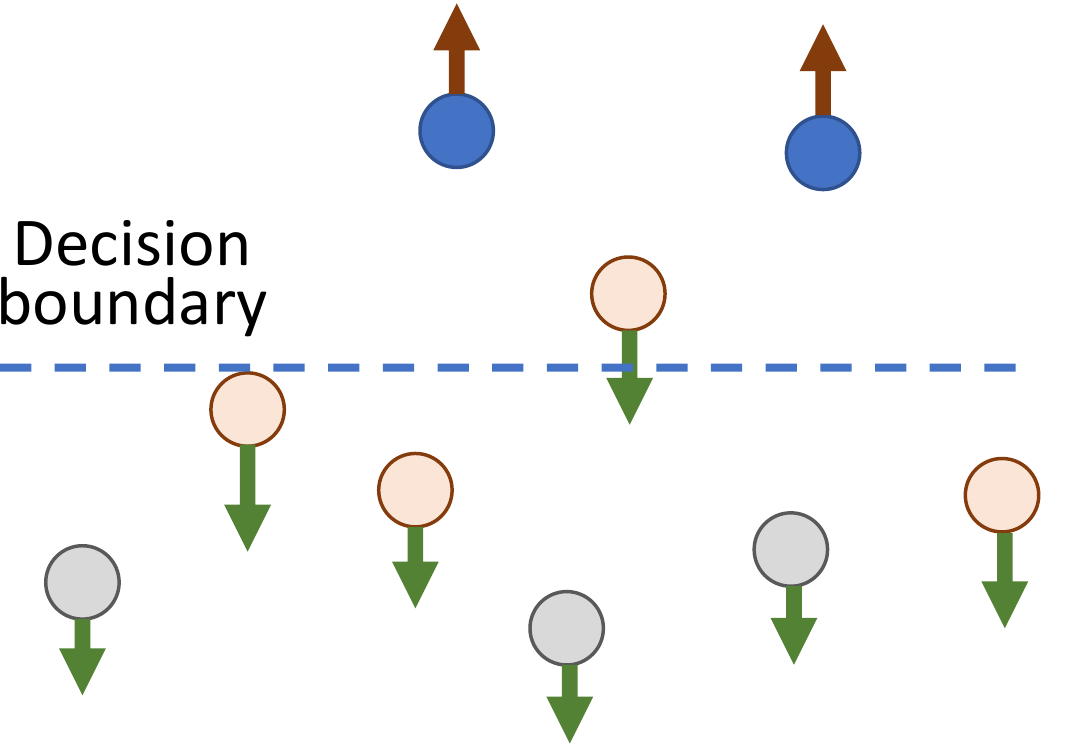}
   				\caption{Without the KG}
   				\label{fig:d1}
			\end{subfigure}
			\hfill
			\begin{subfigure}[b]{0.19\textwidth}
				\includegraphics[width=\textwidth]{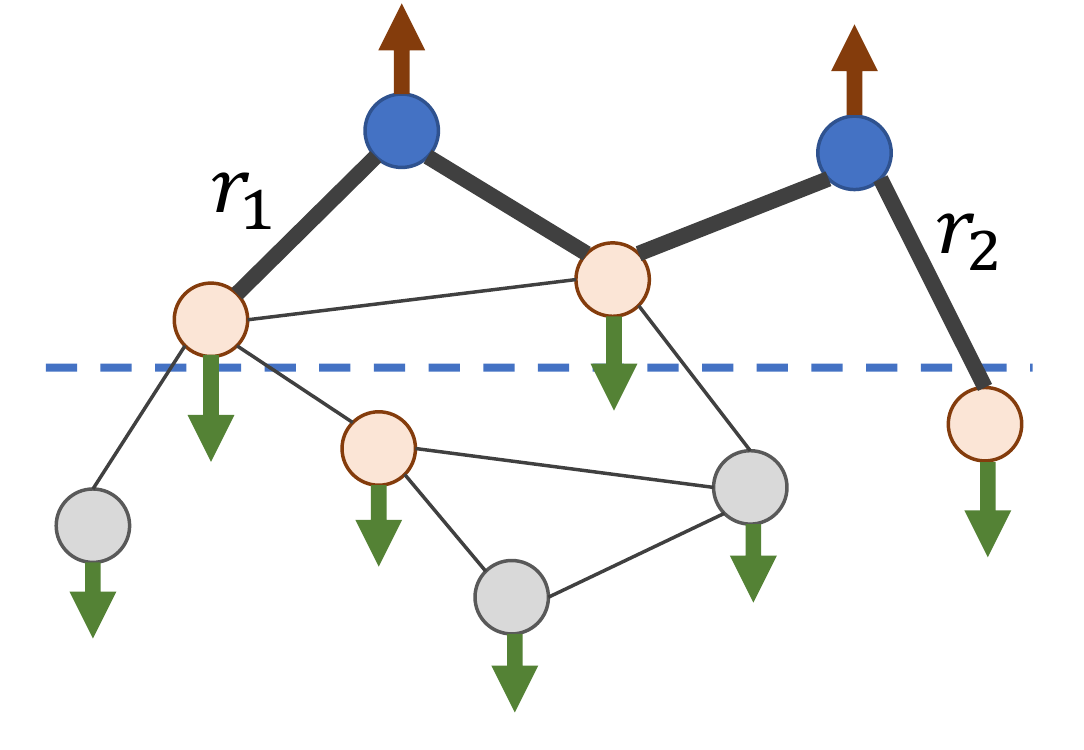}
				\caption{$L = 1$}
				\label{fig:d2}
			\end{subfigure}
			\hfill
			\begin{subfigure}[b]{0.19\textwidth}
   				\includegraphics[width=\textwidth]{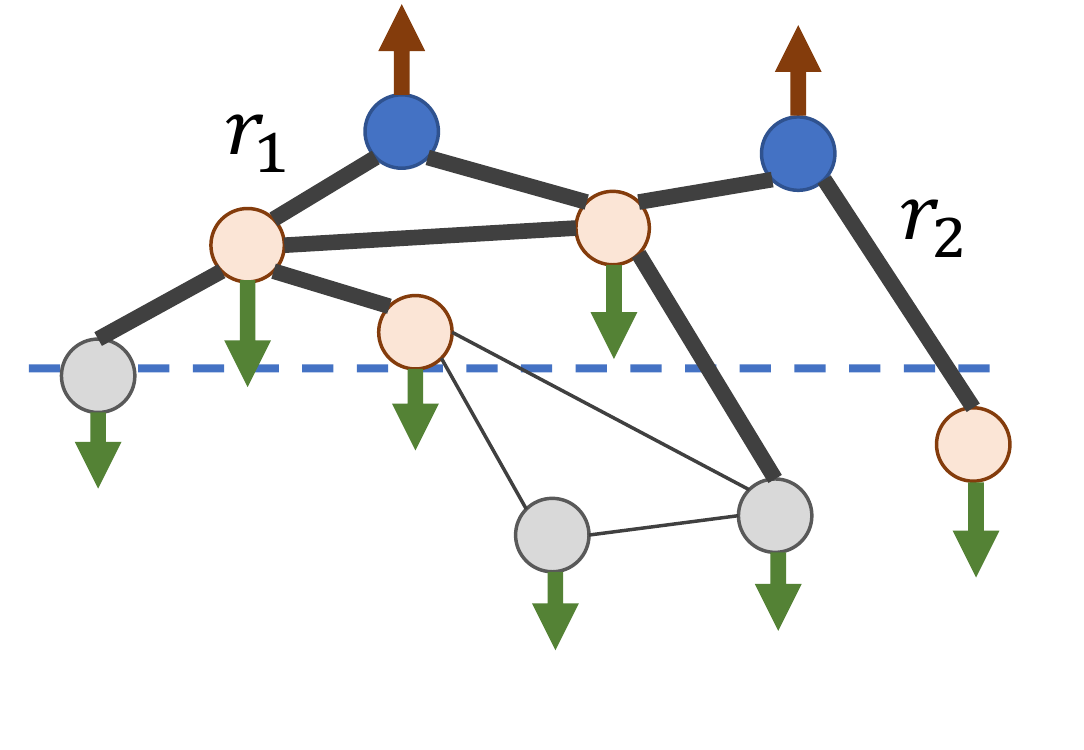}
   				\caption{$L = 2$}
   				\label{fig:d3}
			\end{subfigure}
			\hfill
			\begin{subfigure}[b]{0.19\textwidth}
				\includegraphics[width=\textwidth]{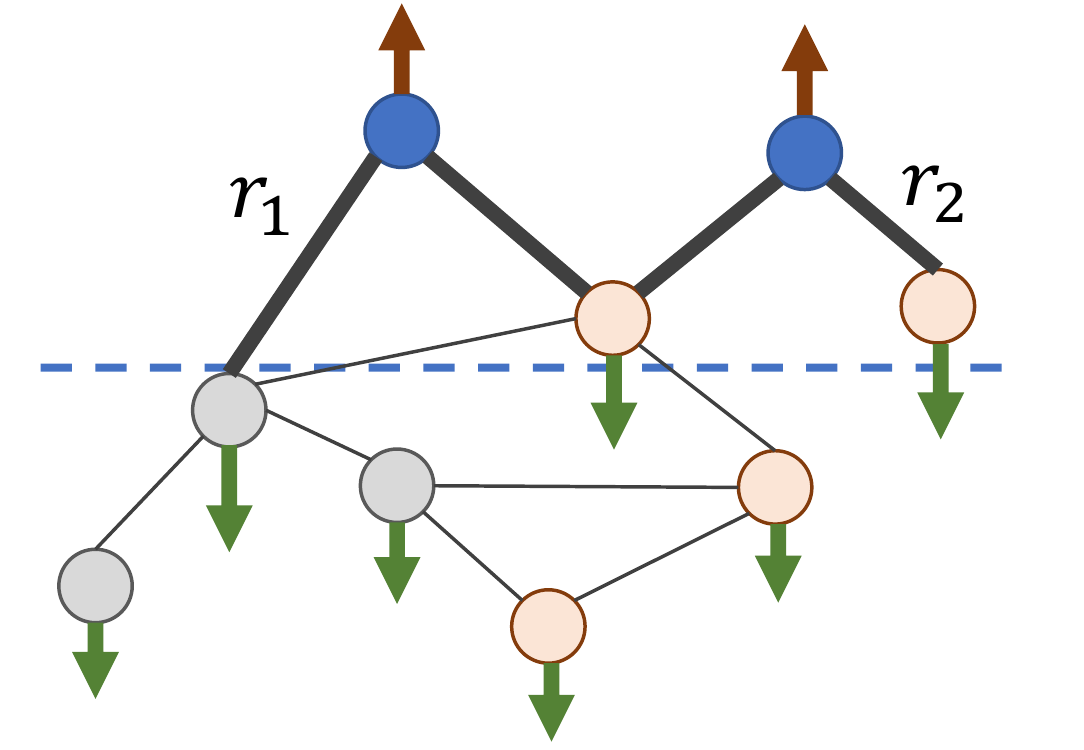}
				\caption{$L = 1$ for another user}
				\label{fig:d4}
			\end{subfigure}
			\hfill
			\begin{subfigure}[b]{0.19\textwidth}
				\includegraphics[width=\textwidth]{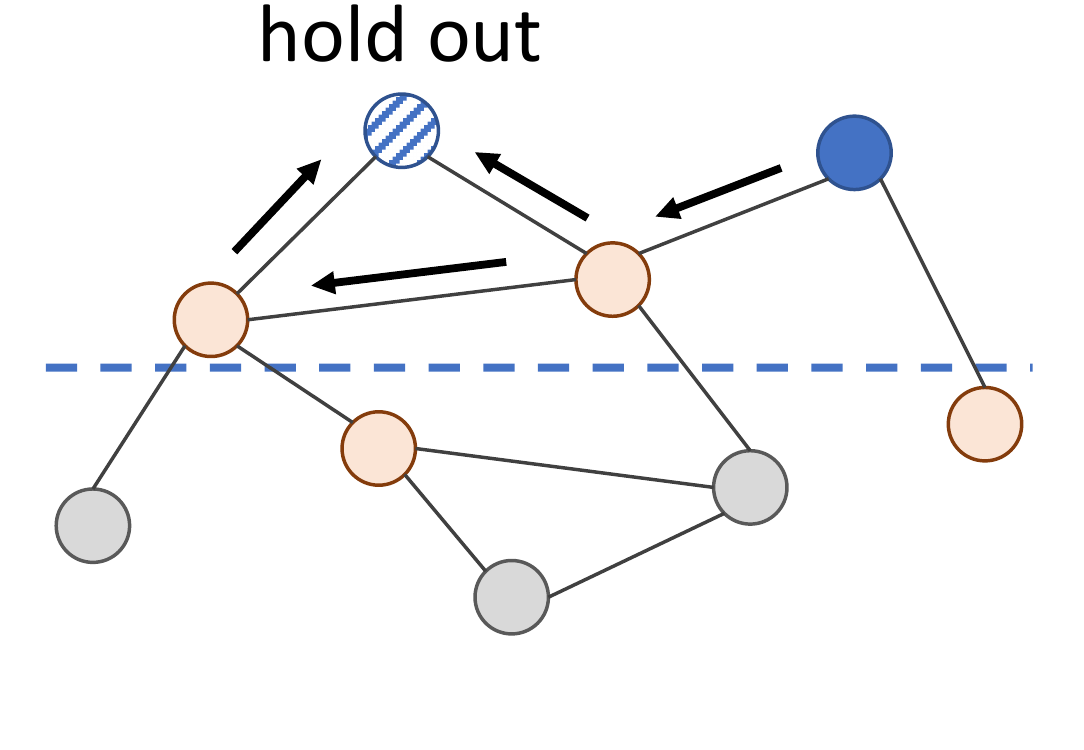}
				\caption{LS regularization}
				\label{fig:d5}
			\end{subfigure}
			\caption{(a) Analogy of a physical equilibrium model for recommender systems; (b)-(d) Illustration of the effect of the KG; (e) Illustration of the effect of label smoothness regularization.}
			\label{fig:discussion}
		\end{figure*}
		
	\subsection{Label Smoothness Regularization}
		It is worth noticing a significant difference between our model and GNNs: In traditional GNNs, edge weights of the input graph are fixed; but in our model, edge weights ${\bf D}_u^{-1/2} {\bf A}_u {\bf D}_u^{-1/2}$ in Eq. (\ref{eq:kgcn}) are learnable (including possible parameters of function $g$ and feature vectors of users and relations) and also requires supervised training like $\bf W$.
		Though enhancing the fitting ability of the model, this will inevitably make the optimization process prone to overfitting, since the only source of supervised signal is from user-item interactions outside GNN layers.
		Moreover, edge weights do play an essential role in representation learning on graphs, as highlighted by a large amount of prior works \cite{zhu2003semi,zhang2007hyperparameter,wang2008label,karasuyama2013manifold,velickovic2017graph}.
		Therefore, more regularization on edge weights is needed to assist the learning of entity representations and to help generalize to unobserved interactions more efficiently.
		
		Let's see how an ideal set of edge weights should be like.
		Consider a real-valued label function $l_u: \mathcal E \rightarrow \mathbb R$ on $\mathcal G$, which is constrained to take a specific value $l_u(v) = y_{uv}$ at node $v \in \mathcal V \subseteq \mathcal E$.
		In our context, $l_u(v) = 1$ if user $u$ finds the item $v$ relevant and has engaged with it, otherwise $l_u(v) = 0$.
		Intuitively, we hope that adjacent entities in the KG are likely to have similar relevancy labels, which is known as \textit{label smoothness assumption}.
		This motivates our choice of energy function $E$:
		\begin{equation}
		\label{eq:smoothness}
			E(l_u, {\bf A}_u) = \frac{1}{2} \sum_{e_i \in \mathcal E, e_j \in \mathcal E} A_u^{ij} \left( l_u(e_i) - l_u(e_j) \right)^2.
		\end{equation}
		We show that the minimum-energy label function is \textit{harmonic} by the following theorem:
		
		\begin{theorem}
		\label{thm:1}
			The minimum-energy label function
			\begin{equation}
				l_u^* = \mathop{\arg\min}_{l_u: l_u(v) = y_{uv}, \forall v \in \mathcal V} E(l_u, {\bf A}_u)
			\end{equation}
			w.r.t. Eq. (\ref{eq:smoothness}) is harmonic, i.e., $l_u^*$ satisfies
			\begin{equation}
				l_u^*(e_i) = \frac{1}{D_u^{ii}} \sum_{e_j \in \mathcal E} A_u^{ij} \ l_u^*(e_j), \ \forall e_i \in \mathcal E \backslash \mathcal V.
			\end{equation}
		\end{theorem}
		
		\begin{proof}
			Taking the derivative of the following equation
			\begin{equation*}
				E(l_u, {\bf A}_u) = \frac{1}{2} \sum_{i, j} A_u^{ij} \left( l_u(e_i) - l_u(e_j) \right)^2
			\end{equation*}
			with respect to $l_u(e_i)$ where $e_i \in \mathcal E \backslash \mathcal V$, we have
			\begin{equation*}
				\frac{\partial E(l_u, {\bf A}_u)}{\partial l_u(e_i)} = \sum_{j} A_u^{ij} \left( l_u(e_i) - l_u(e_j) \right).
			\end{equation*}
			
			The minimum-energy label function $l_u^*$ should satisfy that
			\begin{equation*}
				\frac{\partial E(l_u, {\bf A}_u)}{\partial l_u(e_i)} \bigg |_{l_u = l_u^*} = 0.
			\end{equation*}
			Therefore, we have
			\begin{equation*}
				l_u^*(e_i) = \frac{1}{\sum_{j} A_u^{ij}} \sum_{j} A_u^{ij} \ l_u^*(e_j) = \frac{1}{D_u^{ii}} \sum_{j} A_u^{ij} \ l_u^*(e_j), \ \forall e_i \in \mathcal E \backslash \mathcal V.
			\end{equation*}
		\end{proof}
		
		The harmonic property indicates that the value of $l_u^*$ at each non-item entity $e_i \in \mathcal E \backslash \mathcal V$ is the average of its neighboring entities, which leads to the following label propagation scheme \cite{zhu2005semi}:
		
		\begin{theorem}
		\label{thm:2}
			Repeating the following two steps:
			\begin{enumerate}
				\item Propagate labels for all entities: $l_u(\mathcal E) \leftarrow {\bf D}_u^{-1} {\bf A}_u l_u(\mathcal E)$, where $l_u(\mathcal E)$ is the vector of labels for all entities;
				\item Reset labels of all items to initial labels: $l_u(\mathcal V) \leftarrow {\bf Y}[u, \mathcal V]^\top$, where $l_u(\mathcal V)$ is the vector of labels for all items and ${\bf Y}[u, \mathcal V]=[y_{uv_1}, y_{uv_2}, \cdots]$ are initial labels;
			\end{enumerate}
			will lead to $l_u \rightarrow l_u^*$.
		\end{theorem}
		
		\begin{proof}
			Let $l_u(\mathcal E) = \begin{bmatrix}l_u(\mathcal V)\\[0.3 em]l_u(\mathcal E \backslash \mathcal V)\end{bmatrix}$.			Since $l_u(\mathcal V)$ is fixed on ${\bf Y}[u, \mathcal V]$, we are only interested in $l_u(\mathcal E \backslash \mathcal V)$.
			We denote ${\bf P} = {\bf D}_u^{-1} {\bf A}_u$ (the subscript $u$ is omitted from $\bf P$ for ease of notation), and partition matrix $\bf P$ into sub-matrices according to the partition of $l_u$:
			\begin{equation*}
				{\bf P} = \begin{bmatrix}
					{\bf P}_{VV} \quad {\bf P}_{VE}\\
					{\bf P}_{EV} \quad {\bf P}_{EE}\\
				\end{bmatrix}.
			\end{equation*}
			Then the label propagation scheme is equivalent to
			\begin{equation}
			\label{eq:iteration}
				l_u(\mathcal E \backslash \mathcal V) \leftarrow {\bf P}_{EV} {\bf Y}[u, \mathcal V]^\top + {\bf P}_{EE} l_u(\mathcal E \backslash \mathcal V).
			\end{equation}
			Repeat the above procedure, we have
			\begin{equation}
			\label{eq:expand}
				l_u(\mathcal E \backslash \mathcal V) = \lim_{n \rightarrow \infty} ({\bf P}_{EE})^n l_u^{(0)}(\mathcal E \backslash \mathcal V) + \left( \sum_{i=1}^n ({\bf P}_{EE})^{i-1} \right) {\bf P}_{EV} {\bf Y}[u, \mathcal V]^\top,
			\end{equation}
			where $l_u^{(0)}(\mathcal E \backslash \mathcal V)$ is the initial value for $l_u(\mathcal E \backslash \mathcal V)$.
			Now we show that $lim_{n \rightarrow \infty} ({\bf P}_{EE})^n l_u^{(0)}(\mathcal E \backslash \mathcal V) = {\bf 0}$.
			Since $\bf P$ is row-normalized and ${\bf P}_{EE}$ is a sub-matrix of $\bf P$, we have
			\begin{equation*}
				\exists \epsilon < 1, \ \sum_j {\bf P}_{EE} [i, j] \leq \epsilon,
			\end{equation*}
			for all possible row index $i$.
			Therefore,
			\begin{equation*}
			\begin{split}
				\sum_j ({\bf P}_{EE})^n [i, j] &= \sum_j \left( ({\bf P}_{EE})^{(n-1)} {\bf P}_{EE} \right) [i, j]\\
				&= \sum_j \sum_k ({\bf P}_{EE})^{(n-1)}[i, k] \ {\bf P}_{EE} [k, j]\\
				&= \sum_k ({\bf P}_{EE})^{(n-1)}[i, k] \sum_j {\bf P}_{EE} [k, j]\\
				&\leq \sum_k ({\bf P}_{EE})^{(n-1)}[i, k] \ \epsilon\\
				&\leq \cdots \leq \epsilon^n.
			\end{split}
			\end{equation*}
			As $n$ goes infinity, the row sum of $({\bf P}_{EE})^n$ converges to zero, which implies that $({\bf P}_{EE})^n l_u^{(0)}(\mathcal E \backslash \mathcal V) \rightarrow {\bf 0}$.
			It's clear that the choice of initial value $l_u^{(0)}(\mathcal E \backslash \mathcal V)$ does not affect the convergence.
			
			Since $lim_{n \rightarrow \infty} ({\bf P}_{EE})^n l_u^{(0)}(\mathcal E \backslash \mathcal V) = {\bf 0}$, Eq. (\ref{eq:expand}) becomes
			\begin{equation*}
				l_u(\mathcal E \backslash \mathcal V) = \lim_{n \rightarrow \infty} \left( \sum_{i=1}^n ({\bf P}_{EE})^{i-1} \right) {\bf P}_{EV} {\bf Y}[u, \mathcal V]^\top.
			\end{equation*}
			Denote
			\begin{equation*}
				{\bf T} = \lim_{n \rightarrow \infty} \sum_{i=1}^n ({\bf P}_{EE})^{i-1} = \sum_{i=1}^\infty ({\bf P}_{EE})^{i-1},
			\end{equation*}
			and we have
			\begin{equation*}
				{\bf T} - {\bf T} {\bf P}_{EE} = \sum_{i=1}^\infty ({\bf P}_{EE})^{i-1} - \sum_{i=1}^\infty ({\bf P}_{EE})^{i} = {\bf I}.			\end{equation*}
			Therefore, we derive that
			\begin{equation*}
				{\bf T} = (I - {\bf P}_{EE})^{-1},
			\end{equation*}
			and
			\begin{equation*}
				l_u(\mathcal E \backslash \mathcal V) = (I - {\bf P}_{EE})^{-1} {\bf P}_{EV} {\bf Y}[u, \mathcal V]^\top.
			\end{equation*}
			This is the unique fixed point and therefore the unique solution to Eq. (\ref{eq:iteration}).
			Repeating the steps in Theorem \ref{thm:2} leads to
			\begin{equation*}
				l_u(\mathcal E) \rightarrow l_u^*(\mathcal E) =
				\begin{bmatrix}
					{\bf Y}[u, \mathcal V]^\top\\[0.3 em]
					(I - {\bf P}_{EE})^{-1} {\bf P}_{EV} {\bf Y}[u, \mathcal V]^\top
				\end{bmatrix}.
			\end{equation*}
		\end{proof}
		
		Theorem \ref{thm:2} provides a way for reaching the minimum-energy of relevancy label function $E$.
		However, $l_u^*$ does not provide any signal for updating the edge weights matrix ${\bf A}_u$, since the labeled part of $l_u^*$, i.e., $l_u^* (\mathcal V)$, equals their true relevancy labels ${\bf Y}[u, \mathcal V]$;
		Moreover, we do not know true relevancy labels for the unlabeled nodes $l_u^* (\mathcal E \backslash \mathcal V)$.
		
		To solve the issue, we propose minimizing the \textit{leave-one-out} loss \cite{zhang2007hyperparameter}.
		Suppose we hold out a single item $v$ and treat it unlabeled.
		Then we predict its label by using the rest of (labeled) items and (unlabeled) non-item entities.
		The prediction process is identical to label propagation in Theorem \ref{thm:2}, except that the label of item $v$ is hidden and needs to be calculated.
		This way, the difference between the true relevancy label of $v$ (i.e., $y_{uv}$) and the predicted label $\hat l_u(v)$ serves as a supervised signal for regularizing edge weights:
		\begin{equation}
		\label{eq:ls_regularization}
			R({\bf A}) = \sum_{u} R({\bf A}_u) = \sum_{u} \sum_{v} J \big( y_{uv}, \hat l_u(v) \big),
		\end{equation}
		where $J$ is the cross-entropy loss function.
		Given the regularization in Eq. (\ref{eq:ls_regularization}), an ideal edge weight matrix $\bf A$ should reproduce the true relevancy label of each held-out item while also satisfying the smoothness of relevancy labels.

	\subsection{The Unified Loss Function}
		Combining knowledge-aware graph neural networks and LS regularization, we reach the following complete loss function:
		\begin{equation}
		\label{eq:loss}
			\min_{\bf{W}, \bf{A}} \mathcal L = \min_{\bf{W}, \bf{A}} \ \sum_{u, v} J(y_{uv}, \hat y_{uv}) + \lambda R({\bf A}) + \gamma \| \mathcal F \|^2_2,
		\end{equation}
		where $\| \mathcal F \|^2_2$ is the L2-regularizer, $\lambda$ and $\gamma$ are balancing hyper-parameters.
		In Eq. (\ref{eq:loss}), the first term corresponds to the part of GNN that learns the transformation matrix $\bf{W}$ and edge weights $\bf{A}$ simultaneously, while the second term $R(\cdot)$ corresponds to the part of label smoothness that can be seen as adding constraint on edge weights $\bf{A}$.
		Therefore, $R(\cdot)$ serves as regularization on $\bf{A}$ to assist GNN in learning edge weights.
		
		It is also worth noticing that the first term can be seen as \textit{feature propagation} on the KG while the second term $R(\cdot)$ can be seen as \textit{label propagation} on the KG.
		A recommender for a specific user $u$ is actually a mapping from item features to user-item interaction labels, i.e., $\mathcal F_u: {\bf E}_v \rightarrow y_{uv}$ where ${\bf E}_v$ is the feature vector of item $v$.
		Therefore, Eq. (\ref{eq:loss}) utilizes the structural information of the KG on both the feature side and the label side of $\mathcal F_u$ to capture users' higher-order preferences.

\subsection{Discussion}
	\label{sec:discussion}
		How can the knowledge graph help find users' interests?
		To intuitively understand the role of the KG, we make an analogy with a physical equilibrium model as shown in Figure \ref{fig:discussion}.
		Each entity/item is seen as a particle, while the supervised positive user-relevancy signal acts as the force pulling the observed positive items up from the decision boundary and the negative items signal acts as the force pushing the unobserved items down.
		Without the KG (Figure \ref{fig:d1}), these items are only loosely connected with each other through the collaborative filtering effect (which is not drawn here for clarity).
		In contrast, edges in the KG serve as the rubber bands that impose explicit constraints on connected entities.
		When number of layers is $L=1$ (Figure \ref{fig:d2}), representation of each entity is a mixture of itself and its immediate neighbors, therefore, optimizing on the positive items will simultaneously pull their immediate neighbors up together.
		The upward force goes deeper in the KG with the increase of $L$ (Figure \ref{fig:d3}), which helps explore users' long-distance interests and pull up more positive items.
		It is also interesting to note that the proximity constraint exerted by the KG is \textit{personalized} since the strength of the rubber band (i.e., $s_u(r)$) is \textit{user-specific} and \textit{relation-specific}:
		One user may prefer relation $r_1$ (Figure \ref{fig:d2}) while another user (with same observed items but different unobserved items) may prefer relation $r_2$ (Figure \ref{fig:d4}).
		
		Despite the force exerted by edges in the KG, edge weights may be set inappropriately, for example, too small to pull up the unobserved items (i.e., rubber bands are too weak).
		Next, we show by Figure \ref{fig:d5} that how the label smoothness assumption helps regularizing the learning of edge weights.
		Suppose we hold out the positive sample in the upper left and we intend to reproduce its label by the rest of items.
		Since the true relevancy label of the held-out sample is 1 and the upper right sample has the largest label value, the LS regularization term $R({\bf A})$ would enforce the edges with arrows to be large so that the label can ``flow'' from the blue one to the striped one as much as possible.
		As a result, this will tighten the rubber bands (denoted by arrows) and encourage the model to pull up the two upper pink items to a greater extent.
		
	\begin{table}[t]
			\centering
			\begin{tabular}{c|cccc}
				\hline
				& Movie & Book & Music & Restaurant\\
				\hline
				\# users & 138,159 & 19,676 & 1,872 & 2,298,698\\
				\# items & 16,954 & 20,003 & 3,846 & 1,362\\
				\# interactions & 13,501,622 & 172,576 & 42,346 & 23,416,418\\
				\# entities & 102,569 & 25,787 & 9,366 & 28,115\\
				\# relations & 32 & 18 & 60 & 7\\
				\# KG triples & 499,474 & 60,787 & 15,518 & 160,519\\
				\hline
			\end{tabular}
			\vspace{0.05in}
			\caption{Statistics of the four datasets: MovieLens-20M (movie), Book-Crossing (book), Last.FM (music), and Dianping-Food (restaurant).}
			\label{table:statistics}
			\vspace{-0.2in}
		\end{table}

\section{Experiments}
	In this section, we evaluate the proposed KGNN-LS model, and present its performance on four real-world scenarios: movie, book, music, and restaurant recommendations.
	
	\subsection{Datasets}
		We utilize the following four datasets in our experiments for movie, book, music, and restaurant recommendations, respectively, in which the first three are public datasets and the last one is from Meituan-Dianping Group.
		We use Satori\footnote{https://searchengineland.com/library/bing/bing-satori}, a commercial KG built by Microsoft, to construct sub-KGs for MovieLens-20M, Book-Crossing, and Last.FM datasets.
		The KG for Dianping-Food dataset is constructed by the internal toolkit of Meituan-Dianping Group.
		Further details of datasets are provided in Appendix A.
		\begin{itemize}
			\item \textbf{MovieLens-20M}\footnote{\url{https://grouplens.org/datasets/movielens/}} is a widely used benchmark dataset in movie recommendations, which consists of approximately 20 million explicit ratings (ranging from 1 to 5) on the MovieLens website. The corresponding KG contains 102,569 entities, 499,474 edges and 32 relation-types.
			\item \textbf{Book-Crossing}\footnote{\url{http://www2.informatik.uni-freiburg.de/~cziegler/BX/}} contains 1 million ratings (ranging from 0 to 10) of books in the Book-Crossing community.  The corresponding KG contains 25,787 entities, 60,787 edges and 18 relation-types.
			\item \textbf{Last.FM}\footnote{\url{https://grouplens.org/datasets/hetrec-2011/}} contains musician listening information from a set of 2 thousand users from Last.fm online music system.  The corresponding KG contains 9,366 entities, 15,518 edges and 60 relation-types.
			\item \textbf{Dianping-Food} is provided by Dianping.com\footnote{\url{https://www.dianping.com/}}, which contains over 10 million interactions (including clicking, buying, and adding to favorites) between approximately 2 million users and 1 thousand restaurants. The corresponding KG contains 28,115 entities, 160,519 edges and 7 relation-types.
		\end{itemize}
		
		The statistics of the four datasets are shown in Table \ref{table:statistics}.

		\begin{figure}[t]
		\centering
        \begin{subfigure}[b]{0.22\textwidth}
            \includegraphics[width=\textwidth]{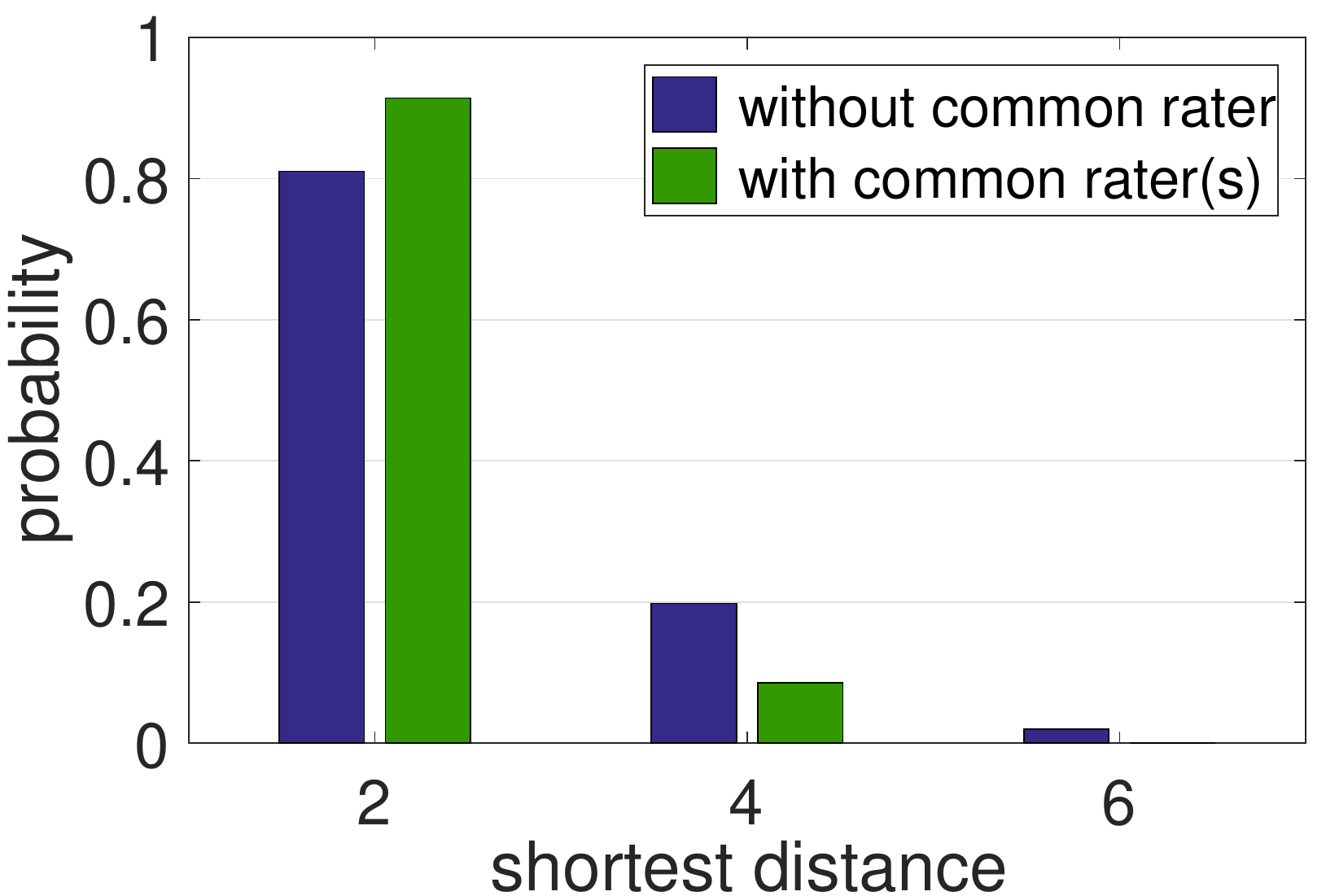}
            \caption{MovieLens-20M}
            \label{fig:es_1}
        \end{subfigure}
        \hfill
        \begin{subfigure}[b]{0.22\textwidth}
            \includegraphics[width=\textwidth]{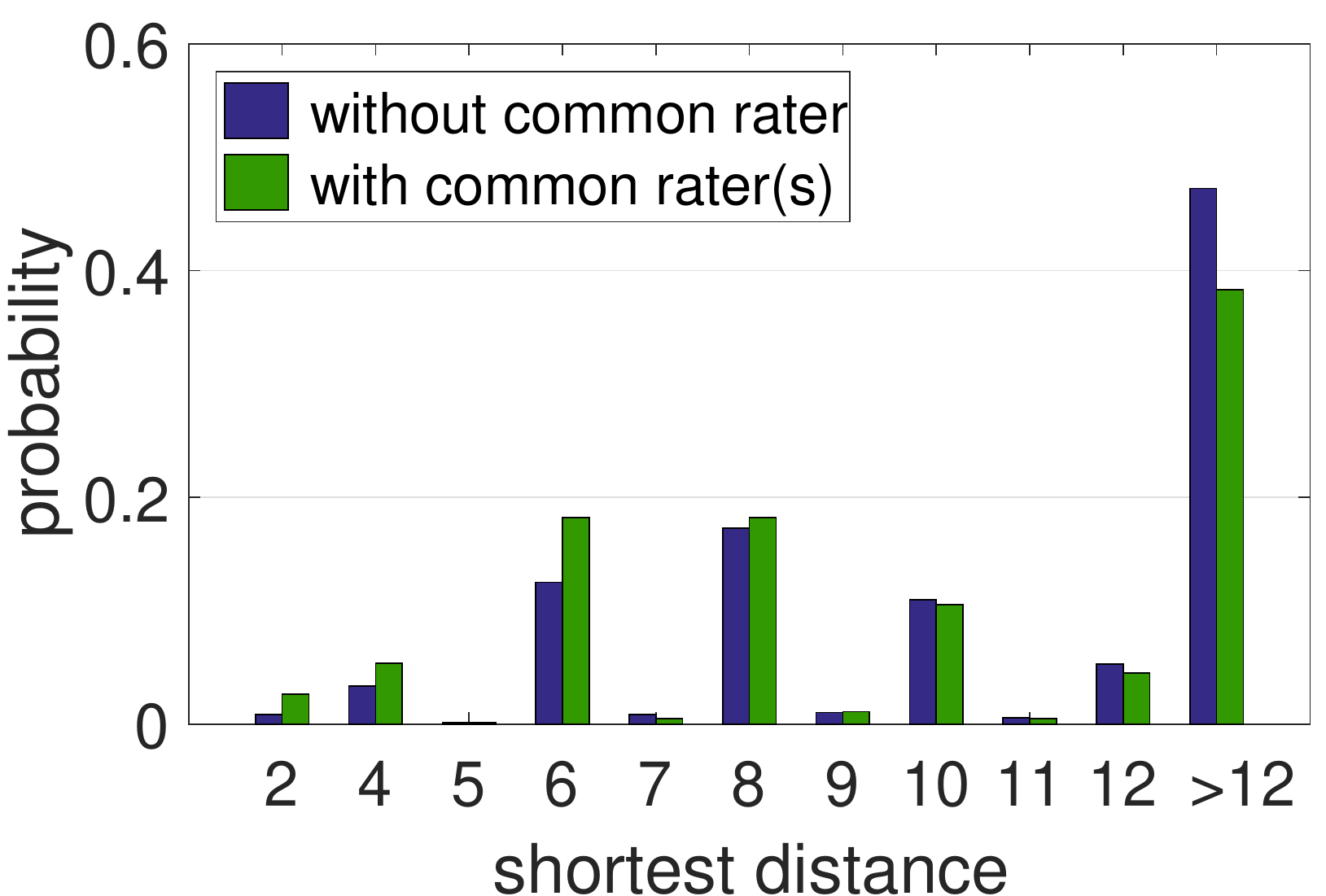}
            \caption{Last.FM}
            \label{fig:es_2}
        \end{subfigure}
        \caption{Probability distribution of the shortest path distance between two randomly sampled items in the KG under the circumstance that (1) they have no common user in the dataset; (2) they have common user(s) in the dataset.}
        \label{fig:es}
    	\end{figure}
    	
    		\begin{table*}[t]
    			\centering
    			\setlength{\tabcolsep}{2pt}
    			\begin{tabular}{c|cccc|cccc|cccc|cccc}
    				\hline
    				\multirow{2}{*}{Model} & \multicolumn{4}{c|}{MovieLens-20M} & \multicolumn{4}{c|}{Book-Crossing} & \multicolumn{4}{c|}{Last.FM} & \multicolumn{4}{c}{Dianping-Food} \\
            		& \multicolumn{1}{c}{\textit{R@2}} & \multicolumn{1}{c}{\textit{R@10}} & \multicolumn{1}{c}{\textit{R@50}} & \multicolumn{1}{c|}{\textit{R@100}} & \multicolumn{1}{c}{\textit{R@2}} & \multicolumn{1}{c}{\textit{R@10}} & \multicolumn{1}{c}{\textit{R@50}} & \multicolumn{1}{c|}{\textit{R@100}} & \multicolumn{1}{c}{\textit{R@2}} & \multicolumn{1}{c}{\textit{R@10}} & \multicolumn{1}{c}{\textit{R@50}} & \multicolumn{1}{c|}{\textit{R@100}} & \multicolumn{1}{c}{\textit{R@2}} & \multicolumn{1}{c}{\textit{R@10}} & \multicolumn{1}{c}{\textit{R@50}} & \multicolumn{1}{c}{\textit{R@100}} \\
            		\hline
            		SVD & 0.036 & 0.124 & 0.277 & 0.401 & 0.027 & 0.046 & 0.077 & 0.109 & 0.029 & 0.098 & 0.240 & 0.332 & 0.039 & 0.152 & 0.329 & 0.451 \\
            		LibFM & 0.039 & 0.121 & 0.271 & 0.388 & 0.033 & 0.062 & 0.092 & 0.124 & 0.030 & 0.103 & 0.263 & 0.330 & 0.043 & 0.156 & 0.332 & 0.448 \\
            		LibFM + TransE & 0.041 & 0.125 & 0.280 & 0.396 & 0.037 & 0.064 & 0.097 & 0.130 & 0.032 & 0.102 & 0.259 & 0.326 & 0.044 & 0.161 & \textbf{0.343} & 0.455 \\
            		PER & 0.022 & 0.077 & 0.160 & 0.243 & 0.022 & 0.041 & 0.064 & 0.070 & 0.014 & 0.052 & 0.116 & 0.176 & 0.023 & 0.102 & 0.256 & 0.354 \\
            		CKE & 0.034 & 0.107 & 0.244 & 0.322 & 0.028 & 0.051 & 0.079 & 0.112 & 0.023 & 0.070 & 0.180 & 0.296 & 0.034 & 0.138 & 0.305 & 0.437 \\
            		RippleNet & \textbf{0.045} & 0.130 & 0.278 & 0.447 & 0.036 & 0.074 & 0.107 & 0.127 & 0.032 & 0.101 & 0.242 & 0.336 & 0.040 & 0.155 & 0.328 & 0.440 \\
            		\hline
            		KGNN-LS & 0.043 & \textbf{0.155} & \textbf{0.321} & \textbf{0.458} & \textbf{0.045} & \textbf{0.082} & \textbf{0.117} & \textbf{0.149} & \textbf{0.044} & \textbf{0.122} & \textbf{0.277} & \textbf{0.370} & \textbf{0.047} & \textbf{0.170} & 0.340 & \textbf{0.487} \\
            		\hline
				\end{tabular}
				\vspace{0.05in}
				\caption{The results of $Recall@K$ in top-K recommendation.}
				\label{table:topk}
				\vspace{-0.1in}
			\end{table*}
			
			\begin{table}[t]
    			\centering
    			\setlength{\tabcolsep}{6pt}
    			\begin{tabular}{c|cccc}
    				\hline
    				Model & Movie & Book & Music & Restaurant \\
            		\hline
            		SVD & 0.963 & 0.672 & 0.769 & 0.838 \\
            		LibFM & 0.959 & 0.691 & 0.778 & 0.837 \\
            		LibFM + TransE & 0.966 & 0.698 & 0.777 & 0.839 \\
            		PER & 0.832 & 0.617 & 0.633 & 0.746 \\
            		CKE & 0.924 & 0.677 & 0.744 & 0.802 \\
            		RippleNet & 0.960 & 0.727 & 0.770 & 0.833 \\
            		\hline
            		KGNN-LS & \textbf{0.979} & \textbf{0.744} & \textbf{0.803} & \textbf{0.850} \\
            		\hline
				\end{tabular}
				\vspace{0.05in}
				\caption{The results of $AUC$ in CTR prediction.}
				\label{table:ctr}
				\vspace{-0.2in}
			\end{table}
		
	\subsection{Baselines}
		We compare the proposed KGNN-LS model with the following baselines for recommender systems, in which the first two baselines are KG-free while the rest are all KG-aware methods.
		The hyper-parameter setting of KGNN-LS is provided in Appendix B.
		\begin{itemize}
			\item
				\textbf{SVD} \cite{koren2008factorization} is a classic CF-based model using inner product to model user-item interactions.
				We use the unbiased version (i.e., the predicted engaging probability is modeled as $y_{uv} = {\bf u}^\top {\bf v}$).
				The dimension and learning rate for the four datasets are set as: $d = 8$, $\eta = 0.5$ for MovieLens-20M, Book-Crossing; $d = 8$, $\eta = 0.1$ for Last.FM; $d = 32$, $\eta = 0.1$ for Dianping-Food.
			\item
				\textbf{LibFM} \cite{rendle2012factorization} is a widely used feature-based factorization model for CTR prediction.
				We concatenate user ID and item ID as input for LibFM.
				The dimension is set as $\{1, 1, 8\}$ and the number of training epochs is $50$ for all datasets.
			\item
				\textbf{LibFM + TransE} extends LibFM by attaching an entity representation learned by TransE \cite{bordes2013translating} to each user-item pair.
				The dimension of TransE is $32$ for all datasets.
			\item
				\textbf{PER} \cite{yu2014personalized} is a representative of path-based methods, which treats the KG as heterogeneous information networks and extracts meta-path based features to represent the connectivity between users and items.
				We use manually designed ``user-item-attribute-item'' as meta-paths, i.e., ``user-movie-director-movie'', ``user-movie-genre-movie'', and ``user-movie-star-movie'' for MovieLens-20M; ``user-book-author-book'' and ``user-book-genre-book'' for Book-Crossing, ``user-musician-date\_of\_birth-musician'' (date of birth is discretized), ``user-musician-country-musician'', and ``user-musician-genre-musician'' for Last.FM; ``user-restaurant-dish-restaurant'', ``user-restaurant-business\_area-restaurant'', ``user-restaurant-tag-restaurant'' for Dianping-Food.
				The settings of dimension and learning rate are the same as SVD.
			\item
				\textbf{CKE} \cite{zhang2016collaborative} is a representative of embedding-based methods, which combines CF with structural, textual, and visual knowledge in a unified framework.
				We implement CKE as CF plus a structural knowledge module in this paper.
				The dimension of embedding for the four datasets are $64$, $128$, $64$, $64$.
				The training weight for KG part is $0.1$ for all datasets.
				The learning rate are the same as in SVD.
			\item
				\textbf{RippleNet} \cite{wang2018ripple} is a representative of hybrid methods, which is a memory-network-like approach that propagates users' preferences on the KG for recommendation.
				For RippleNet, $d=8$, $H=2$, $\lambda_1 = 10^{-6}$, $\lambda_2=0.01$, $\eta=0.01$ for MovieLens-20M; $d=16$, $H=3$, $\lambda_1 = 10^{-5}$, $\lambda_2=0.02$, $\eta=0.005$ for Last.FM; $d=32$, $H=2$, $\lambda_1 = 10^{-7}$, $\lambda_2=0.02$, $\eta=0.01$ for Dianping-Food.
		\end{itemize}

	\subsection{Validating the Connection between $\mathcal G$ and $\bf Y$}
    	To validate the connection between the knowledge graph $\mathcal G$ and user-item interaction $\bf Y$, we conduct an empirical study where we investigate the correlation between the shortest path distance of two randomly sampled items in the KG and whether they have common user(s) in the dataset, that is there exist user(s) that interacted with both items.
		For MovieLens-20M and Last.FM, we randomly sample ten thousand item pairs that have no common users and have at least one common user, respectively, then count the distribution of their shortest path distances in the KG.
		The results are presented in Figure \ref{fig:es}, which clearly show that \textit{if two items have common user(s) in the dataset, they are likely to be more close in the KG}.
		For example, if two movies have common user(s) in MovieLens-20M, there is a probability of $0.92$ that they will be within 2 hops in the KG, while the probability is $0.80$ if they have no common user.
		This finding empirically demonstrates that exploiting the proximity structure of the KG can assist making recommendations.
		This also justifies our motivation to use label smoothness regularization to help learn entity representations.

    		\begin{figure*}
  				\begin{minipage}[t]{0.3\linewidth}    				\centering 
    				\includegraphics[width=\textwidth]{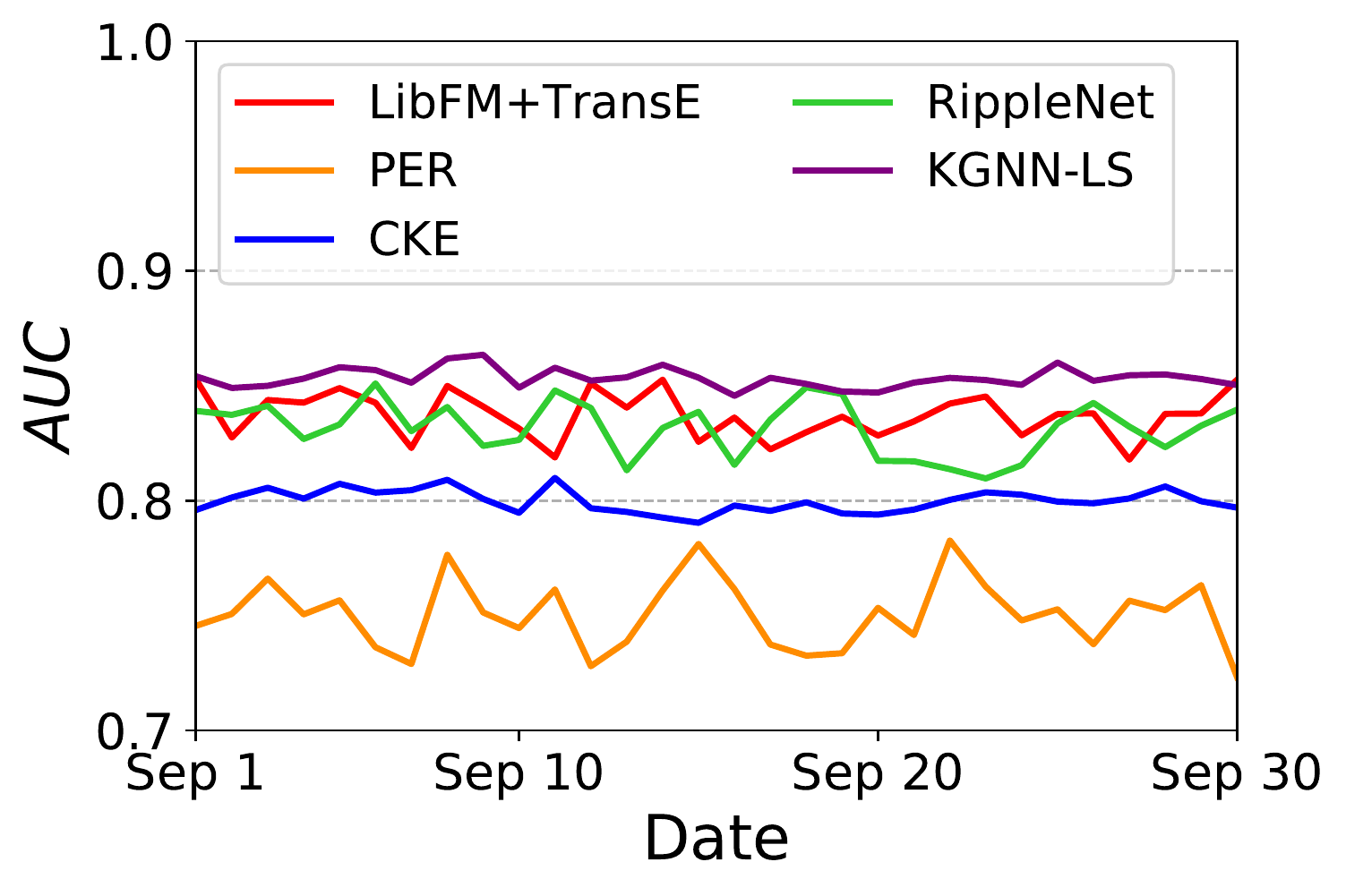}
    				\vspace{-0.2in}
    				\caption{Daily $AUC$ of all methods on Dianping-Food in September 2018.} 
    				\label{fig:food} 
  				\end{minipage}
  				\hfill
  				\begin{minipage}[t]{0.3\linewidth} 
    				\centering 
    				\includegraphics[width=\textwidth]{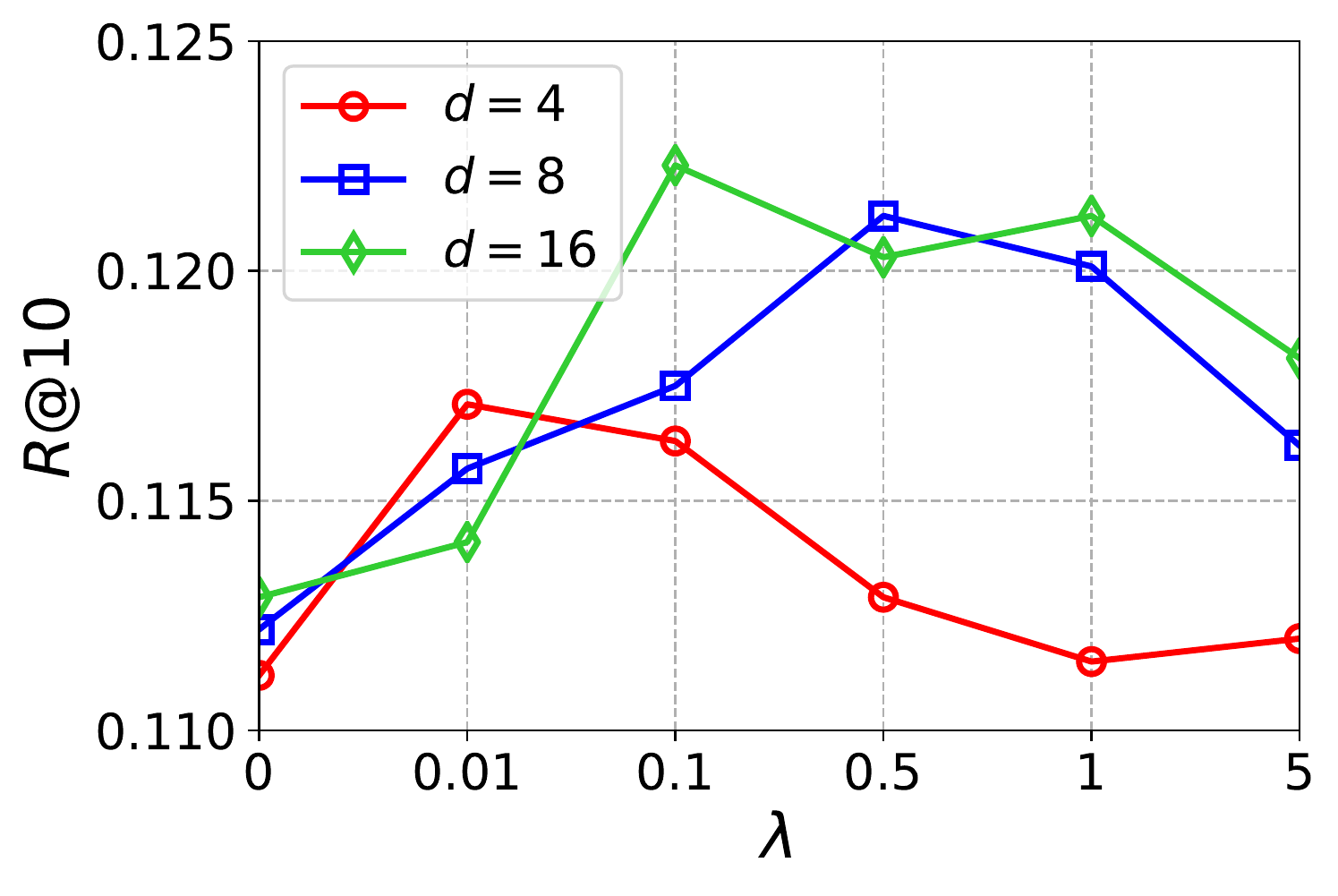}
    				\vspace{-0.2in}
    				\caption{Effectiveness of LS regularization on Last.FM.} 
    				\label{fig:ls} 
  				\end{minipage}
  				\hfill
  				\begin{minipage}[t]{0.315\linewidth} 
    				\centering 
    				\includegraphics[width=\textwidth]{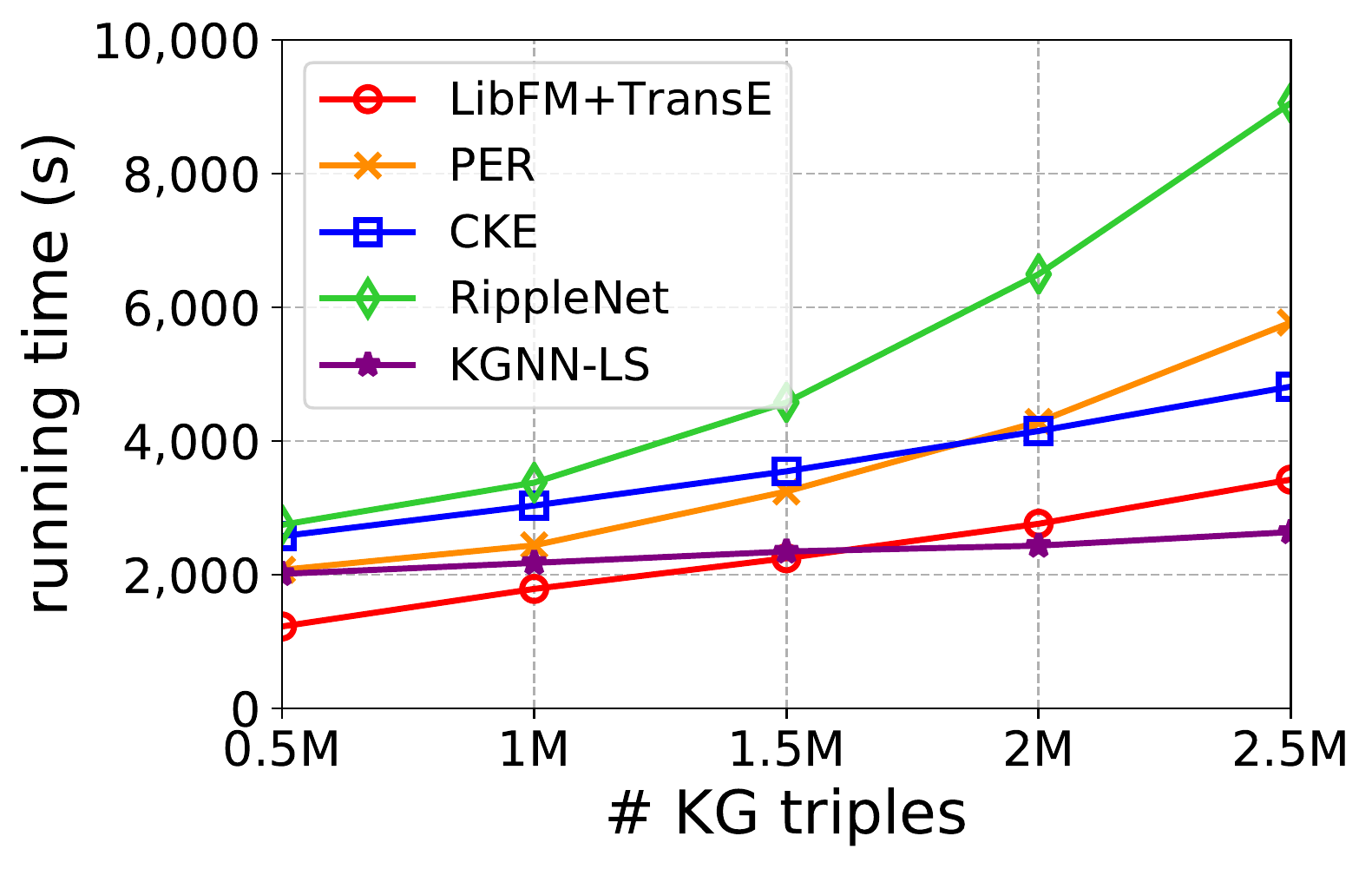}
    				\vspace{-0.2in}
    				\caption{Running time of all methods w.r.t. KG size on MovieLens-20M.} 
    				\label{fig:scalability} 
  				\end{minipage} 
			\end{figure*}
		
	\subsection{Results}
		\subsubsection{Comparison with Baselines}			We evaluate our method in two experiment scenarios:
			(1) In top-$K$ recommendation, we use the trained model to select $K$ items with highest predicted click probability for each user in the test set, and choose $Recall@K$ to evaluate the recommended sets.
			(2) In click-through rate (CTR) prediction, we apply the trained model to predict each piece of user-item pair in the test set (including positive items and randomly selected negative items).
			We use $AUC$ as the evaluation metric in CTR prediction.
			
			The results of top-$K$ recommendation and CTR prediction are presented in Tables \ref{table:topk} and \ref{table:ctr}, respectively, which show that KGNN-LS outperforms baselines by a significant margin.
			For example, the $AUC$ of KGNN-LS surpasses baselines by $5.1\%$, $6.9\%$, $8.3\%$, and $4.3\%$ on average in MovieLens-20M, Book-Crossing, Last.FM, and Dianping-Food datasets, respectively.
        				
			We also show daily performance of KGNN-LS and baselines on Dianping-Food to investigate performance stability.
			Figure \ref{fig:food} shows their $AUC$ score from September 1, 2018 to September 30, 2018.
			We notice that the curve of KGNN-LS is consistently above baselines over the test period;
			Moreover, the performance of KGNN-LS is also with low variance, which suggests that KGNN-LS is also robust and stable in practice.

		\subsubsection{Effectiveness of LS Regularization}
			Is the proposed LS regularization helpful in improving the performance of GNN?
			To study the effectiveness of LS regularization, we fix the dimension of hidden layers as $4$, $8$, and $16$, then vary $\lambda$ from $0$ to $5$ to see how performance changes.
			The results of $R@10$ in Last.FM dataset are plotted in Figure \ref{fig:ls}.
			It is clear that the performance of KGNN-LS with a non-zero $\lambda$ is better than $\lambda=0$ (the case of Wang et al. \cite{wang2019knowledge}), which justifies our claim that LS regularization can assist learning the edge weights in a KG and achieve better generalization in recommender systems.
			But note that a too large $\lambda$ is less favorable, since it overwhelms the overall loss and misleads the direction of gradients.
			According to the experiment results, we find that a $\lambda$ between $0.1$ and $1.0$ is preferable in most cases.

		\subsubsection{Results in cold-start scenarios}
			\begin{table}[t]
				\centering
				\setlength{\tabcolsep}{6pt}
				\begin{tabular}{c|ccccc}
					\hline
					$r$ & $20\%$ & $40\%$ & $60\%$ & $80\%$ & $100\%$ \\
					\hline
					SVD & 0.882 & 0.913 & 0.938 & 0.955 & 0.963 \\
					LibFM & 0.902 & 0.923 & 0.938 & 0.950 & 0.959 \\
					LibFM+TransE & 0.914 & 0.935 & 0.949 & 0.960 & 0.966 \\
					PER & 0.802 & 0.814 & 0.821 & 0.828 & 0.832 \\
					CKE & 0.898 & 0.910 & 0.916 & 0.921 & 0.924 \\
					RippleNet & 0.921 & 0.937 & 0.947 & 0.955 & 0.960 \\
					KGNN-LS & \textbf{0.961} & \textbf{0.970} & \textbf{0.974} & \textbf{0.977} & \textbf{0.979} \\
					\hline
				\end{tabular}
				\vspace{0.05in}
				\caption{$AUC$ of all methods w.r.t. the ratio of training set $r$.}
				\label{table:sparse}
				\vspace{-0.2in}
			\end{table}
		
			One major goal of using KGs in recommender systems is to alleviate the sparsity issue.
			To investigate the performance of KGNN-LS in cold-start scenarios, we vary the size of training set of MovieLens-20M from $r=100\%$ to $r=20\%$ (while the validation and test set are kept fixed), and report the results of $AUC$ in Table \ref{table:sparse}.
			When $r=20\%$, $AUC$ decreases by $8.4\%$, $5.9\%$, $5.4\%$, $3.6\%$, $2.8\%$, and $4.1\%$ for the six baselines compared to the model trained on full training data ($r=100\%$), but the performance decrease of KGNN-LS is only $1.8\%$.
			This demonstrates that KGNN-LS still maintains predictive performance even when user-item interactions are sparse.

		\subsubsection{Hyper-parameters Sensitivity}
			We first analyze the sensitivity of KGNN-LS to the number of GNN layers $L$.
			We vary $L$ from $1$ to $4$ while keeping other hyper-parameters fixed.
			The results are shown in Table \ref{table:L}.
			We find that the model performs poorly when $L=4$, which is because a larger $L$ will mix too many entity embeddings in a given entity, which \textit{over-smoothes} the representation learning on KGs.
			KGNN-LS achieves the best performance when $L = 1$ or $2$ in the four datasets.
			
			\begin{table}[t]
				\centering
				\setlength{\tabcolsep}{8pt}
				\begin{tabular}{c|cccc}
					\hline
					$L$ & 1 & 2 & 3 & 4\\
					\hline
					MovieLens-20M & \textbf{0.155} & 0.146 & 0.122 & 0.011 \\
					Book-Crossing & 0.077 & \textbf{0.082} & 0.043 & 0.008 \\
					Last.FM & \textbf{0.122} & 0.106 & 0.105 & 0.057 \\
					Dianping-Food & 0.165 & \textbf{0.170} & 0.061 & 0.036 \\
					\hline
				\end{tabular}
				\vspace{0.05in}
				\caption{$R@10$ w.r.t. the number of layers $L$.}
				\label{table:L}
				\vspace{-0.2in}
			\end{table}
			
			We also examine the impact of the dimension of hidden layers $d$ on the performance of KGNN-LS.
			The result in shown in Table \ref{table:d}.
			We observe that the performance is boosted with the increase of $d$ at the beginning, because more bits in hidden layers can improve the model capacity.
			However, the performance drops when $d$ further increases, since a too large dimension may overfit datasets.
			The best performance is achieved when $d = 8 \sim 64$.
			
			\begin{table}[t]
				\centering
				\setlength{\tabcolsep}{4pt}
				\begin{tabular}{c|cccccc}
					\hline
					$d$ & 4 & 8 & 16 & 32 & 64 & 128\\
					\hline
					MovieLens-20M & 0.134 & 0.141 & 0.143 & \textbf{0.155} & \textbf{0.155} & 0.151 \\
					Book-Crossing & 0.065 & 0.073 & 0.077 & 0.081 & \textbf{0.082} & 0.080 \\
					Last.FM & 0.111 & 0.116 & \textbf{0.122} & 0.109 & 0.102 & 0.107 \\
					Dianping-Food & 0.155 & \textbf{0.170} & 0.167 & 0.166 & 0.163 & 0.161 \\
					\hline
				\end{tabular}
				\vspace{0.05in}
				\caption{$R@10$ w.r.t. the dimension of hidden layers $d$.}
				\label{table:d}
				\vspace{-0.2in}
			\end{table}

	\subsection{Running Time Analysis}
		We also investigate the running time of our method with respect to the size of KG.
		We run experiments on a Microsoft Azure virtual machine with 1 NVIDIA Tesla M60 GPU, 12 Intel Xeon CPUs (E5-2690 v3 @2.60GHz), and 128GB of RAM.
		The size of the KG is increased by up to five times the original one by extracting more triples from Satori, and the running times of all methods on MovieLens-20M are reported in Figure \ref{fig:scalability}.
		Note that the trend of a curve matters more than the real values, since the values are largely dependent on the minibatch size and the number of epochs (yet we did try to align the configurations of all methods).
		The result show that KGNN-LS exhibits strong scalability even when the KG is large.

\section{Conclusion and Future Work}
	In this paper, we propose knowledge-aware graph neural networks with label smoothness regularization for recommendation.
	KGNN-LS applies GNN architecture to KGs by using user-specific relation scoring functions and aggregating neighborhood information with different weights.	In addition, the proposed label smoothness constraint and leave-one-out loss provide strong regularization for learning the edge weights in KGs.
	We also discuss how KGs benefit recommender systems and how label smoothness can assist learning the edge weights.
	Experiment results show that KGNN-LS outperforms state-of-the-art baselines in four recommendation scenarios and achieves desirable scalability with respect to KG size.
	
	In this paper, LS regularization is proposed for recommendation task with KGs.
	It is interesting to examine the LS assumption on other graph tasks such as link prediction and node classification.
	Investigating the theoretical relationship between feature propagation and label propagation is also a promising direction.

\vspace{1em}
\noindent \textbf{Acknowledgements}.
This research has been supported in part by NSF OAC-1835598, DARPA MCS, ARO MURI, Boeing, Docomo, Hitachi, Huawei, JD, Siemens, and Stanford Data Science Initiative.

\bibliographystyle{ACM-Reference-Format}
\bibliography{reference}

\newpage
\section*{Appendix}
	\subsection*{A \ \ \ Additional Details on Datasets}
		MovieLens-20M, Book-Crossing, and Last.FM dataset contain explicit feedbacks data (Last.FM provides the listening count as weight for each user-item interaction).
		Therefore, we transform them into implicit feedback, where each entry is marked with 1 indicating that the user has rated the item positively.
		The threshold of positive rating is 4 for MovieLens-20M, while no threshold is set for Book-Crossing and Last.FM due to their sparsity.
		Additionally, we randomly sample an unwatched set of items and mark them as 0 for each user, the number of which equals his/her positively-rated ones.
		
		We use Microsoft Satori to construct the KGs for MovieLens-20M, Book-Crossing, and Last.FM dataset.
		In one triple in Satori KG, the head and tail are either IDs or textual content, and the relation is with the form ``domain.head\_category.tail\_category'' (e.g., ``book.book.author'').
		We first select a subset of triples from the whole Satori KG with a confidence level greater than 0.9.
		Given the sub-KG, we collect Satori IDs of all valid movies/books/musicians by matching their names with tail of triples \textit{(head, film.film.name, tail)}, \textit{(head, book.book.title, tail)}, or \textit{(head, type.object.name, tail)}, for the three datasets.
		Items with multiple matched or no matched entities are excluded for simplicity.
		After having the set of item IDs, we match these item IDs with the head of all triples in Satori sub-KG, and select all well-matched triples as the final KG for each dataset.
		
		Dianping-Food dataset is collected from Dianping.com, a Chinese group buying website hosting consumer reviews of restaurants similar to Yelp.
		We select approximately 10 million interactions between users and restaurants in Dianping.com from May 1, 2015 to December 12, 2018.
		The types of positive interactions include clicking, buying, and adding to favorites, and we sample negative interactions for each user.
		The KG for Dianping-Food is collected from Meituan Brain, an internal knowledge graph built for dining and entertainment by Meituan-Dianping Group.
		The types of entities include POI (restaurant), city, first-level and second-level category, star, business area, dish, and tag;
		The types of relations correspond to the types of entities (e.g., ``organization.POI.has\_dish'').

	\subsection*{B \ \ \ Additional Details on Hyper-parameter Searching}
		In KGNN-LS, we set functions $g$ and $f$ as inner product, $\sigma$ as \textit{ReLU} for non-last-layers and $tanh$ for the last-layer.
		Note that the size of neighbors of an entity in a KG may vary significantly over the KG.
		To keep the computation more efficient, we uniformly sample a fixed-size set of neighbors for each entity instead of using its full set of neighbors.
		The number of sampled neighbors for each entity is denoted by $S$.
		Hyper-parameter settings are given in Table \ref{table:hp}, which are determined by optimizing $R@10$ on a validation set.
		The search spaces for hyper-parameters are as follows:
		\begin{itemize}
			\item $S = \{2, 4, 8, 16, 32\}$;
			\item $d = \{4, 8, 16, 32, 64, 128\}$;
			\item $L = \{1, 2, 3, 4\}$;
			\item $\lambda = \{0, 0.01, 0.1, 0.5, 1, 5\}$;
			\item $\gamma = \{ 10^{-9}, 10^{-8}, 10^{-7}, 2 \times 10^{-7}, 5 \times 10^{-7}, 10^{-6}, 2 \times 10^{-6}, 5 \times 10^{-6}, 10^{-5}, 2 \times 10^{-5}, 5 \times 10^{-5}, 10^{-4}, 2 \times 10^{-4}, 5 \times 10^{-4}, 10^{-3} \}$;
			\item $\eta = \{10^{-5}, 2 \times 10^{-5}, 5 \times 10^{-5}, 10^{-4}, 2 \times 10^{-4}, 5 \times 10^{-4}, 10^{-3}, 2 \times 10^{-3}, 5 \times 10^{-3}, 10^{-2}, 2 \times 10^{-2}, 5 \times 10^{-2}, 10^{-1}\}$.
		\end{itemize}
		For each dataset, the ratio of training, validation, and test set is $6:2:2$.
		Each experiment is repeated $5$ times, and the average performance is reported.
		All trainable parameters are optimized by Adam algorithm.
		The code of KGNN-LS is implemented with Python 3.6, TensorFlow 1.12.0, and NumPy 1.14.3.
		
		\begin{table}[t]
			\centering
			\begin{tabular}{c|cccc}
				\hline
				& Movie & Book & Music & Restaurant\\
				\hline
				$S$ & 16 & 8 & 8 & 4\\
				$d$ & 32 & 64 & 16 & 8\\
				$L$ & 1 & 2 & 1 & 2\\
				$\lambda$ & 1.0 & 0.5 & 0.1 & 0.5 \\
				$\gamma$ & $10^{-7}$ & $2 \times 10^{-5}$ & $10^{-4}$ & $10^{-7}$\\
				$\eta$ & $2 \times 10^{-2}$ & $2 \times 10^{-4}$ & $5 \times 10^{-4}$ & $2 \times 10^{-2}$\\
				\hline
			\end{tabular}
			\vspace{0.05in}
			\caption{Hyper-parameter settings for the four datasets ($S$: number of sampled neighbors for each entity; $d$: dimension of hidden layers, $L$: number of layers, $\lambda$: label smoothness regularizer weight, $\gamma$: L2 regularizer weight, $\eta$: learning rate).}
			\label{table:hp}
			\vspace{-0.2in}
		\end{table}

\end{document}